\documentclass{article} 

\usepackage[accepted]{icml2014} 
\usepackage{url}
\usepackage{natbib}

\usepackage{amsfonts}
\usepackage{amssymb,amsmath, amsthm}
\usepackage{epsfig}
\newtheorem{lemma}{Lemma}
\newtheorem{theorem}{Theorem}
\newtheorem{definition}{Definition}

\usepackage{wrapfig}
\usepackage[tight]{subfigure}

\floatname{algorithm}{Algorithm}
\renewcommand{\algorithmicrequire}{\textbf{Input:}}
\renewcommand{\algorithmicensure}{\textbf{Output:}}

\def\E{\mathbb{E}}

\def\R{\mathbb{R}}

\newcommand{\alekhcomment}[1]{\textbf{ALEKH:} #1}

\newcommand{\link}{\ensuremath{g}}
\newcommand{\linkloss}{\ensuremath{\Phi}}
\newcommand{\linkopt}{\ensuremath{\link}}
\newcommand{\wopt}{\ensuremath{w^*}}
\newcommand{\Wopt}{\ensuremath{W^*}}
\newcommand{\lossn}{\ensuremath{\ell_n}}
\newcommand{\order}{\ensuremath{\mathcal{O}}}
\newcommand{\ip}[2]{\ensuremath{\left \langle #1, #2 \right \rangle}}
\newcommand{\norm}[1]{\ensuremath{\| #1 \|}}
\newcommand{\covhat}{\ensuremath{\widehat{\Sigma}}}
\newcommand{\linkbasis}{\ensuremath{G}}
\newcommand{\yiter}[1]{\ensuremath{\hat y^{(#1)}}}
\newcommand{\ziter}[1]{\ensuremath{\tilde y^{(#1)}}}
\newcommand{\xweights}[1]{\ensuremath{\xwplain_{#1}}}
\newcommand{\calweights}[1]{\ensuremath{\calwplain_{#1}}}
\newcommand{\yitermat}[1]{\ensuremath{\hat Y^{(#1)}}}
\newcommand{\xwplain}{\ensuremath{W}}
\newcommand{\calwplain}{\ensuremath{\tilde W}}
\newcommand{\Wstar}{\ensuremath{W^{*}}}
\newcommand{\cond}{\ensuremath{\kappa}}
\newcommand{\ynoclip}[1]{\ensuremath{\bar{y}^{{#1}}}}
\newcommand{\ones}{\ensuremath{1\!\!1}}
\newcommand{\zitermat}[1]{\ensuremath{\tilde Y^{{#1}}}}
\newcommand{\wtil}{\ensuremath{\widetilde{w}}}
\newcommand{\Wtil}{\ensuremath{\widetilde{W}}}

\newenvironment{proof-of-theorem}[1][{}]{\noindent{\bf Proof of Theorem {#1}}
  \hspace*{1em}}{\qed\smallskip\\}

\DeclareMathOperator*{\argmin}{arg\,min}

\author{
PRE-PRINT: PLEASE DO NOT DISTRIBUTE
\AND
    Alekh Agarwal\thanks{ Use footnote for providing further information
about author (webpage, alternative address)---\emph{not} for acknowledging
funding agencies.} \\
Microsoft Research\\
New York, NY \\
\texttt{alekha@microsoft.com} \\
\And
Sham M. Kakade\\
Microsoft Research\\
Cambridge, MA\\
\texttt{skakade@microsoft.com} \\
\And
Nikos Karampatziakis\\
Microsoft Cloud and Information Services Lab. \\
Redmond, WA 98052 \\
\texttt{nikosk@microsoft.com} \\
\And
Le Song\\
College of Computing, Georgia Tech \\
Atlanta, Georgia \\
\texttt{lsong@cc.gatech.edu} \\
\And
Gregory Valiant\\
Microsoft Research\\
Cambridge, MA\\
\texttt{gregory.valiant@gmail.com}\\
}

\icmltitlerunning{Least Squares Revisited}

\begin{document}

\twocolumn[
  
\icmltitle{Least Squares Revisited: \\
Scalable Approaches for Multi-class Prediction}

\icmlauthor{Alekh Agarwal}{alekha@microsoft.com}
\icmladdress{Microsoft Research New York, NY}
\icmlauthor{Sham M. Kakade}{skakade@microsoft.com}
\icmladdress{Microsoft Research Cambridge, MA}
\icmlauthor{Nikos Karampatziakis}{nikosk@microsoft.com}
\icmladdress{Microsoft Cloud and Information Services Lab. Redmond, WA 98052}
\icmlauthor{Le Song}{lsong@cc.gatech.edu}
\icmladdress{College of Computing, Georgia Tech, Atlanta, Georgia}
\icmlauthor{Gregory Valiant}{valiant@stanford.edu}
\icmladdress{Computer Science Department Stanford University}

\vskip 0.3in
]

\begin{abstract}
    This work provides simple algorithms for multi-class (and
    multi-label) prediction in settings where both the number of
    examples $n$ and the data dimension $d$ are relatively
    large. These robust and parameter free algorithms are essentially
    iterative least-squares updates and very versatile both in theory
    and in practice. On the theoretical front, we present several
    variants with convergence guarantees. Owing to their effective use
    of second-order structure, these algorithms are substantially
    better than first-order methods in many practical scenarios. On
    the empirical side, we present a scalable stagewise variant of our
    approach, which achieves dramatic computational speedups over
    popular optimization packages such as Liblinear and Vowpal Wabbit
    on standard datasets (MNIST and CIFAR-10), while attaining
    state-of-the-art accuracies.
\end{abstract}

\section{Introduction}

\setlength{\abovedisplayskip}{4pt}
\setlength{\abovedisplayshortskip}{1pt}
\setlength{\belowdisplayskip}{4pt}
\setlength{\belowdisplayshortskip}{1pt}
\setlength{\jot}{3pt}

\setlength{\textfloatsep}{3ex}

The aim of this paper is to develop robust and scalable algorithms for
multi-class classification problems with $k$ classes, where the number
of examples $n$ and the number of features $d$ is simultaneously quite
large. Typically, such problems have been approached by the
minimization of a convex surrogate loss, such as the multiclass
hinge-loss or the multiclass logistic loss, or reduction to convex
binary subproblems such as one-versus-rest. Given the size of the
problem, (batch or online) first-order methods are typically the
methods of choice to solve these underlying optimization
problems. First-order updates, usually linear in the dimension in
their computational complexity, easily scale to large $d$. To deal
with the large number of examples, online methods are particularly
appealing in the single machine setting, while batch methods are often
preferred in distributed settings.

Empirically however, these first-order approaches are often found to
be lacking. Many natural high-dimensional data such as images, audio,
and video typically result in ill-conditioned optimization
problems. While each iteration of a first-order method is fast, the
number of iterations needed unavoidably scale with the condition
number of the data matrix \citep{NemYu83}, even for simple generalized
linear models (henceforth GLM). Hence, the convergence of these
methods is still rather slow on many real-world datasets with decaying
spectrum.

A natural alternative in such scenarios is to use second-order
methods, which are robust to the conditioning of the
data. In this paper, we present simple second-order methods for
multiclass prediction in GLMs. Crucially, the methods are parameter
free, robust in practice and admit easy extensions. As an example, we
show a more sophisticated variant which learns the unknown link
function in the GLM simultaneously with the weights. Finally, we also
present a practical variant to tackle the difficulties typically
encountered in applying second-order methods to high-dimensional
problems. We develop a block-coordinate descent style stagewise
regression procedure that incrementally solves least-squares problems
on small batches of features. The result of this overall development
is a suite of techniques that are simple, versatile and substantially
faster than several other state-of-the-art optimization methods.

{\bf Our Contributions: } Our work has three main
contributions. Working in the GLM framework: $ \E[y \mid x] = g(W x)
$, where $y$ is a \emph{vector} of predictions, $W$ is the weight
matrix, and $g$ is the vector valued link function, we present a
simple second-order update rule. The update is based on a majorization
of the Hessian, and uses a scaled version of the empirical second
moment $\frac1n \sum_i x_i x_i^T$ as the preconditioner. Our algorithm
is parameter-free and does not require a line search for
convergence. Furthermore our computations only involve a $d\times d$
matrix unlike IRLS and other Hessian related approaches where matrices
are $\order(dk\times dk)$ for multiclass problems\footnote{This is a
  critical distinction as we focus on tasks involving increasingly
  complex class hierarchies, particularly in the context of computer
  vision problems.}. Theoretically, the proposed method enjoys an
iteration complexity independent of the condition number of the data
matrix as an immediate observation.

We extend our algorithm to simultaneously estimate the weights as well
as the link function in GLMs under a parametric assumption on the link
function, building on ideas from isotonic regression~\cite{KS09,
  KakadeKKS11}. We provide a global convergence guarantee for this
algorithm despite the non-convexity of the problem. To the best of our
knowledge, this is the first work to formulate and address the problem
of isotonic regression in the multiclass classification
setting. Practically this enables, for example, the use of our current
predictions as features in order to improve the predictions in
subsequent iterations. Similar procedures are common for binary SVMs
~\cite{Platt99probabilisticoutputs} and for
re-ranking\cite{collins2000discriminative}.

Both the above algorithms, despite being metric free, still scale
somewhat poorly with the dimensionality of the problem owing to the
quadratic cost of the representation and updates. To address this
problem, we take a cue from ideas in block-coordinate descent and
stagewise regression literature. Specifically, we choose a subset of
the features and perform one of the above second-order updates on that
subset only. We then repeat this process, successively fitting the
residuals. We demonstrate excellent empirical performance of this
procedure on two tasks: MNIST and CIFAR-10. In settings where the
second order information is relevant, such as MNIST and CIFAR-10, we
find that stagewise variants can be highly effective, providing orders
of magnitude speed-ups over online methods and other first-order
approaches. This is particularly noteworthy since we compare a simple
MATLAB implementation of our algorithms with sophisticated C software
for the alternative approaches. In contrast, for certain text problems
where the data matrix is well conditioned, online methods are highly
effective. Notably, we also achieve state of the art accuracy results
on MNIST and CIFAR-10, outperforming the ``dropout'' neural
net~\cite{hinton2012improving}, where our underlying optimization
procedures are entirely based on simple least squares
approaches. These promising results highlight that this is a fruitful
avenue for the development of further theory and algorithms, which we
leave for future work.

{\bf Related Work:} There is much work on scalable algorithms for
large, high-dimensional datasets. A large chunk of this work builds on
and around online learning and stochastic optimization,
leveraging the ability of these algorithms to ensure a very rapid
initial reduction of test error (see
e.g.~\cite{bottou08tradeoff,Shalev-Shwartz12}). These methods can be
somewhat unsuited though, when optimization to a relatively high
precision is desired, for example, when the data matrix is
ill-conditioned and small changes in the parameters can lead to large
changes in the outputs. This has led to interesting works on hybrid
methods that interpolate between an initial online and subsequent
batch behavior~\cite{ShwartzZh2012, LeRouxScBa12}. There has also been
a renewed interest in Quasi-Newton methods scalable to statistical
problems using stochastic approximation
ideas~\cite{ByrdChNeNo11,bordes-bottou-gallinari-2009}. High-dimensional
problems have also led to natural consideration of block coordinate
descent style procedures, both in serial~\cite{Nesterov12} and
distributed~\cite{RichtarikTa12,RechtReWrNi11} settings. Indeed, in
some of our text experiments, our stagewise procedure comes
quite close to a block-coordinate descent type update.  There are also
related approaches for training SVMs that extract the most information
out of a small subset of data before moving to the next
batch \cite{Chapelle2007,MatsushimaViSm12, YuHsChLi12}.

On the statistical side, our work most directly generalizes past works
on learning in generalized linear models for binary classification,
when the link function is known or unknown~\cite{KS09,KakadeKKS11}. A
well-known case where squared loss was used in conjunction with a
stagewise procedure to fit binary and multi-class GLMs is the gradient
boosting machine~\cite{friedman2001greedy}. In the statistics
literature, the iteratively reweighed least squares algorithm (IRLS)
is the workhorse for fitting GLMs and also
works by recasting the optimization problem to a series of least
squares problems. However, IRLS can (and does in some cases) diverge, while the
proposed algorithms are guaranteed to make progress on each
iteration. Another difficulty with IRLS (also shared by some
majorization algorithms e.g.,~\cite{jebara2012majorization}) is that
each iteration needs to work with a new Hessian since it depends on
the parameters. In contrast, our algorithms use the same matrix
throughout their run.

\section{Setting and Algorithms}

We begin with the simple case of binary GLMs, 
before addressing the more challenging multi-class setting.

\subsection{Warmup: Binary GLMs}

The canonical definition of a GLM in binary
classification (where $y\in \{0,1\}$) setup posits the probabilistic model
\begin{equation}
  \E[y \mid x] = \linkopt({\wopt}^T x),
  \label{eqn:glm-binary}
\end{equation}
where $\linkopt~:~ \R \mapsto \R$ is a monotone increasing function,
and $\wopt \in \R^d$. To facilitate the development of better
algorithms, assume that
$\linkopt$ is a $L$-Lipschitz function of its univariate argument.
Since $\linkopt$ is a monotone increasing univariate function, there
exists a convex function $\linkloss~:~ \R \mapsto \R$ such that
$\linkloss' = \linkopt$. Based on this convex function, let us
define a convex loss function.

\begin{definition}[Calibrated loss]
  Given the GLM~\eqref{eqn:glm-binary}, define
  the associated convex loss
\begin{equation}
  \ell(w;(x,y)) = \linkloss(w^T x) - yw^T x.
  \label{eqn:glm-binary-loss}
\end{equation}
\end{definition}

Up to constants independent of $w$, this definition yields the
least-squares loss for the identity link function, $\link(u) = u$, and
the logistic loss for the logit link function, $\link(u) =
e^u/(1+e^u)$. The loss is termed calibrated: for each $x$,
minimizing the above loss yields a consistent estimate of the weights
$\wopt$.  Precisely,

\begin{lemma}
  Suppose $g$ is a monotone function and that Eq.~\eqref{eqn:glm-binary}
  holds. Then ${\wopt}$ is a minimizer of $\E[\ell(w;(x,y))]$, where
  the expectation is with respect to the distribution on $x$ and
  $y$. Furthermore, any other minimizer $\tilde w$ of
  $\E[\ell(w;(x,y))]$ (if one such exists) also satisfies $ \E[y \mid
  x] = \linkopt({\tilde w}^T x)$.
\end{lemma}

\begin{proof}
First, let us show that ${\wopt}$ is a pointwise minimizer of
$\E[\ell(w;(x,y))|x]$ (almost surely).
$\E[\ell(w;(x,y))]$.  Observe for any point $x$,
\begin{align} \label{eqn:consistent}
  \nonumber &\E[\nabla \ell(w;(x,y)) \mid x] = \E[\nabla
    \linkloss( {w}^T x) - xy | x]\\ 
  &\qquad \qquad =
  \E[\linkopt({w}^T x) x \mid x] - \linkopt({\wopt}^T x) x 
\end{align}
where the second equality follows since $\linkloss' = \linkopt$ and $\E[y
  \mid x] = \linkopt(\wopt x)$ by the probabilistic
model~\eqref{eqn:glm-binary}. Hence, ${\wopt}$ is a global minimizer (since
the loss function is convex).

Now let us show that any other minimizer $\tilde w$ (if one exists)
also satisfies $ \E[y \mid x] = \linkopt({\tilde w}^T x)$.  Let
$\tilde S$ be the set of $x$ such that $\E[y \mid x] =
\linkopt({\tilde w}^T x)$. It suffices to show $\Pr(x\notin \tilde S)=0$.
Suppose this is not the case. We then have:
\begin{eqnarray}\nonumber
&&\E[\ell(\tilde w;(x,y))] \\ 
&=& \Pr(x\in \tilde S)  \ \E[  \ell(\tilde w;(x,y))|x\in \tilde S] \\
&&+
\Pr(x\notin \tilde S)  \  \E[  \ell(\tilde w;(x,y))|x\notin \tilde S] \\
&=& \Pr(x\in \tilde S)  \ \E[  \ell({\wopt};(x,y))|x\in \tilde S] \\
&&+
\Pr(x\notin \tilde S)  \  \E[  \ell(\tilde w;(x,y))|x\notin \tilde S] \\
&>& \Pr(x\in \tilde S)  \ \E[  \ell({\wopt};(x,y))|x\in \tilde S] \\
&&+
\Pr(x\notin \tilde S)  \  \E[  \ell({\wopt};(x,y))|x\notin \tilde S] \\
&=&\E[\ell({\wopt};(x,y))] \\ 
\end{eqnarray}
where the second equality follows by ~\eqref{eqn:consistent} and the
inequality follows since for $x\notin \tilde S$, $\E[\ell(\tilde
w;(x,y))|x] > \E[\ell({\wopt};(x,y))|x]$ (again by
~\eqref{eqn:consistent}, since ${\wopt}$ is a minimizer of
$\E[\ell(w;(x,y))|x]$, almost surely). This contradicts the optimality of
$\tilde w$.
\end{proof}
As another intuition, this loss
corresponds to the negative log-likelihood when the GLM specifies an
exponential family with the sufficient statistic $y$. Similar
observations have been noted for the binary case in some prior works
as well (see ~\citet{KakadeKKS11,RavikumarWaYu2008}). Computing the
optimal $\wopt$ simply amounts to using any standard convex
optimization procedure. We now discuss these choices in the context of
multi-class prediction.

\subsection{Multi-class GLMs and Minimization Algorithms}

The first question in the multi-class case concerns the definition of
a generalized linear model; monotonicity is not immediately extended
in the multi-class setting. Following the definition in the recent
work of~\citet{Agarwal2013}, we extend the binary case by defining the
model:
\begin{equation}
  \E[y \mid x] = \nabla \linkloss(\Wopt x)  := g(\Wopt x) \,
  \label{eqn:glm-multi}
\end{equation}
where $\Wopt \in \R^{k\times d}$ is the weight matrix, $\linkloss~:~
\R^k\mapsto \R$ is a proper and convex lower semicontinuous function
of $k$ variables and $y \in \R^k$ is a vector with 1 for the correct
class and zeros elsewhere. This definition essentially corresponds to
the link function $g=\nabla \linkloss$ satisfying (maximal and
cyclical) monotonicity~\cite{Rockafellar66} (natural extensions of
monotonicity to vector spaces). Furthermore, when the
GLM~\eqref{eqn:glm-multi} corresponds to an exponential family with
sufficient statistics $y$, then $\linkloss$ corresponds to the
log-partition function like the binary case, and is always
convex~\cite{Lauritzen}.

This formulation immediately yields an analogous
definition for a calibrated multi-class loss.
\begin{definition}[Calibrated multi-class loss]
  Given the GLM~\eqref{eqn:glm-multi}, define
  the associated convex loss
  \begin{equation}
    \ell(W;(x,y)) = \linkloss(Wx) - y^T Wx.
    \label{eqn:glm-multi-loss}
  \end{equation}
\end{definition}
Observe that we obtain the
multi-class logistic loss, when the probabilistic
model~\eqref{eqn:glm-multi} is a multinomial logit model.

The loss function is convex as before. It is Fisher
consistent: the minimizer of the expected loss is $\Wopt$ (as in
Equation~\ref{eqn:consistent}). In particular,

\begin{lemma}
Suppose $\linkloss~:~
\R^k\mapsto \R$ is a (proper and lower semicontinuous) convex function
and that Eq.~\eqref{eqn:glm-multi} holds.
Then ${\Wopt}$ is a minimizer of
  $\E[\ell(W;(x,y))]$, where the expectation is with respect to the
  distribution on $x$ and $y$. Furthermore, any other minimizer
  $\tilde W$ of $\E[\ell(W;(x,y))]$ (if one such exists) also satisfies
  $ \E[y \mid x] = \linkopt({\tilde w}^T x)$.
\end{lemma}

The proof is identical to that of before. Again, convexity only
implies that all local minimizers are global minimizers.

As before, existing convex optimization algorithms can be utilized to
estimate the weight matrix $W$. First-order methods applied to the
problem have per-iteration complexity of $\order(dk)$, but can require
a large number of iterations as discussed before. Here, the difficulty
in utilizing second-order approaches is that the Hessian matrix is of
size $dk \times dk$ (e.g. as in IRLS, for logistic regression); any
direct matrix inversion method is now much more computationally expensive
even for moderate sized $k$.

Algorithm~\ref{alg:glm-multi-full} provides a simple variant of least
squares regression --- which repeatedly fits the residual error ---
that exploits the second order structure in $x$. Indeed, as shown in
the appendix, the algorithm uses a block-diagonal upper bound on the
Hessian matrix in order to preserve the correlations between the
covariates $x$, but does not consider interactions across the
different classes to have a more computationally tractable update. The
algorithm has several attractive properties. Notably, (i) the
algorithm is parameter free\footnote{Here and below we refer to
  parameter free algorithms from the point of view of optimization: no
  learning rates, backtracking constants etc. The overall learning
  algorithms may still require setting other parameters, such as the
  regularizer.} and (ii) the algorithm only inverts a $d\times d$
matrix. Furthermore, this matrix is independent of the weights $W$
(and the labels) and can be computed only once ahead of time. In that
spirit, the algorithm can also be viewed as \emph{preconditioned
  gradient descent}, with a block diagonal preconditioner whose
diagonal blocks are identical and equal to the matrix $\covhat^{-1}$.
At each step, we utilize the residual error $\hat\E[(\hat y - y)
  x^T]$, akin to a gradient update on least-squares loss.  Note the
``stepsize'' here is determined by $L$, a parameter entirely dependent
on the loss function and not on the data. For the case of logistic
regression, simply $L=1$ satisfies this Lipchitz
constraint\footnote{Using Gershgorin's circle theorem it is possible
  to show that $L=1/2$ still leads to a valid upper bound on the
  Hessian.  This is tight and achieved by an example whose class
  probabilities under the current model are $(1/2,1/2,0,0,\ldots)$.}.
Also observe that for the square loss, where $L=1$, the generalized
least squares algorithm reduces to least squares (and terminates in
one iteration).



\begin{algorithm}[t]
\begin{algorithmic}
  \REQUIRE Initial weight matrix $W_0$, data $\{(x_i, y_i)\}$,
  Lipschitz constant $L$, link $g=\nabla \linkloss$.
\vspace{0.1in}
  \STATE Define the (vector valued) predictions $\yiter{t}_i = g(W_tx_i)$
  and the empirical expectations:
\begin{align*}
&\covhat = \hat\E [x_i x_i^T] =\frac{1}{n} \sum_{i=1}^n x_i x_i^T \, \\
&\hat\E[(\yiter{t} - y) x^T] = \frac{1}{n} \sum_{i=1}^n
(\yiter{t}_i - y_i) x_i^T
\end{align*}
  \REPEAT
  \STATE Update the weight matrix $W_t$:
  \begin{equation}
    W_{t+1}^T = W_t^T - \frac{1}{L} \ \covhat^{-1} \ \hat\E[(\yiter{t} - y) x^T]
    \label{eqn:glm-multi-newton}
  \end{equation}
 \UNTIL{convergence}
\end{algorithmic}
\caption{Generalized Least Squares}
\label{alg:glm-multi-full}
\end{algorithm}


We now describe the convergence properties of
Algorithm~\ref{alg:glm-multi-full}. The results are stated in terms of
the sample loss
\begin{equation}
  \lossn(w) = \frac{1}{n} \sum_{i=1}^n \ell(W;(x_i, y_i)).
  \label{eqn:lossn}
\end{equation}
The following additional assumptions regarding the link
function $\nabla \linkloss$ are natural for characterizing convergence
rates. Assuming that the link function $g=\nabla \linkloss$ is $L$-Lipschitz
amounts to the condition
\begin{equation}
  \norm{g(u) - g(v)}_2 \leq L \norm{u -
    v}_2, ~\mbox{where}~u,v \in \R^k.
  \label{eqn:lipschitz-multi}
\end{equation}
If we want a linear convergence rate, we must further assume
$\mu$-strong monotonicity, meaning for all $u,v \in \R^k$:
\begin{equation}
  \ip{g(u) - g(v)}{u-v} \geq \mu \norm{u - v}_2^2.
  \label{eqn:multi-link-strong}
\end{equation}


\begin{theorem}
Define $\Wstar = \arg\min_W \lossn(W)$.  Suppose that the link
function $\nabla \linkloss$ is
$L$-Lipschitz~\eqref{eqn:lipschitz-multi}. Using the generalized Least
Squares updates (Algorithm~\ref{alg:glm-multi-full}) with $W_0 = 0$,
then for all $t = 1,2,\ldots$
  \begin{equation*}
    \lossn(W_t) - \lossn(\Wstar) \leq \frac{2L \norm{\Wstar}^2}{t+4}.
  \end{equation*}
  If, in addition, the link function is $\mu$-strongly
  monotone~\eqref{eqn:multi-link-strong} and let $\cond_\linkloss =
  L/\mu$. Then
  \begin{equation*}
    \lossn(W_t) - \lossn(\Wstar) \leq \frac{L}{2}
    \left(\frac{\cond_{\linkloss} - 1}{\cond_{\linkloss} + 1}
    \right)^t \norm{\Wstar}_F^2.
  \end{equation*}
  \label{thm:glm-multi-newton}
\end{theorem}
The proof rests on demonstrating that the block-diagonal matrix formed
by copies of $L\covhat$ provides a majorization of the Hessian matrix,
along with standard results in convex optimization (see
e.g.~\cite{Nesterov04}) and is deferred to the supplement. Also,
observe that the convergence results in
Theorem~\ref{thm:glm-multi-newton} are completely independent of the
conditioning of the data matrix $\covhat$. Indeed they depend only on
the smoothness and strong convexity properties of $\linkloss$ which is
a function we know ahead of time and control. This is the primary
benefit of these updates over first-order updates.

In order to understand these issues better, let us quickly contrast
these results to the analogous ones for gradient descent. In that
case, we get qualitatively similar dependence on the number of
iterations. However, in the case of Lipschitz $\nabla \linkloss$, the
convergence rate is
$\order\left(\frac{L}{t}\sigma_{\max}\left(\frac{XX^T}{n}
\right)\norm{\Wstar}^2\right)$. Under strong monotonicity, the rate
improves to $\order\left(L\left(\frac{\cond_{\linkloss}\cond_{XX^T}
  \,-\, 1}{\cond_{\linkloss}\cond_{XX^T} \,+\,1}
\right)^t\right)$. That is, the convergence rate is slowed down by
factors depending on the singular values of the empirical covariance
in both the cases. Similar comparisons
can also be made for accelerated versions of both our and vanilla
gradient methods.

\subsection{Unknown Link Function for Multi-class}

The more challenging case is when the link function is unknown. This
setting has two main difficulties: the statistical one of how to
restrict the complexity of the class of link functions and the
computational one of efficient estimation (as opposed to local search
or other herutistic methods).

\begin{algorithm}[t]
\begin{algorithmic}
  \REQUIRE Initial weight matrix $W_0$, set of
  calibration functions $G=\{g_1, \ldots g_m \}$
\vspace{0.1in}
  \STATE Initialize the predictions: $  \yiter{0}_i = \xweights{0} x_i $ 
 \REPEAT
  \STATE Fit the residual: 
\begin{align} 
  \nonumber \xweights{t} &= \argmin_{\xwplain} \sum_{i=1}^n \norm{y_i
  - \yiter{t-1}_i - \xwplain x_i}_2^2, \quad \mbox{and}\\ 
  \ziter{t}_i &= \yiter{t-1}_i+\xweights{t} x_i. 
  \label{eqn:alg-calib-multi1}
\end{align}
  \STATE Calibrate the predictions $\ziter{t}$:
\begin{align}
  \nonumber \calweights{t} &= \argmin_{\calwplain} \sum_{i=1}^n 
\norm{y_i
-\calwplain  \linkbasis(\ziter{t}_i)}_2^2, \quad \mbox{and}\\ 
  \yiter{t}_i &= \textrm{clip}(\calweights{t} \linkbasis(\ziter{t})),
  \label{eqn:alg-calib-multi2}
\end{align}
where clip$(v)$ is the Euclidean projection of $v$ onto the
probability simplex in $\R^k$. 
\UNTIL{convergence}
\end{algorithmic}
\caption{Calibrated Least Squares}
\label{alg:glm-multi-calib}
\end{algorithm}

With regards to the former, a natural restriction is to consider
the class of link functions realized as the
derivative of a convex function in $k$-dimensions.
This naturally extends the Isotron algorithm from the binary
case~\cite{KS09}. Unfortunately, this is an extremely rich class; the
sample complexity of estimating a uniformly bounded convex, Lipschitz
function in $k$ dimensions grows exponentially with
$k$~\cite{Bronshtein76}. In our setting, this would imply that the
number of samples needed for a small error in link function estimation
would necessarily scale exponentially in the number of classes, even
with infinite computational resources at our disposal. To avoid this
curse of dimensionality, assume that there is a finite basis
$\linkbasis$ such that $g^{-1}=(\nabla \linkloss)^{-1} \in
\mbox{lin}(\linkbasis)$, $(\nabla \linkloss)^{-1}$ is the funcional
inverse of $\nabla \linkloss$. Without loss of generality, we also
assume that $\linkbasis$ always contains the identity function. We do
not consider the issue of approximation error here. 

Before presenting the algorithm, let us provide some more intuition
about our assumption $g^{-1}=(\nabla \linkloss)^{-1} \in
\mbox{lin}(\linkbasis)$. Clearly the case of $\linkbasis = g^{-1}$ for
a fixed function $g$ puts us in the setting of the previous
section. More generally, let us consider that $G$ is a dictionary of
$p$ functions so that $g^{-1}(y) = \sum_{i=1}^p \wtil_i G_i(y)$. In
the GLM~\eqref{eqn:glm-multi}, this means that we have an overall
linear-like model\footnote{It is not a linear model since the
  statistical noise passes through the functions $G_i$ rather than
  being additive.}
\begin{equation*} 
  \sum_{i=1}^p \Wtil_i G_i(\E[Y | x]) = \Wopt x. 
\end{equation*}
If we let $p = k$ and $G_i(y)$ be the $i_{th}$ class indicator $y_i$,
then the above equation boils down to 
\begin{equation}\label{eq:cca-like}
\Wtil^T\E[Y | x] = \Wopt x,
\end{equation}
meaning that an unknown linear combination of the
class-conditional probabilities is a linear function of the data. More
generally, we consider $G_i$ to also have higher-order monomials such
as $y_i^2$ or $y_i^3$ so that the LHS is some low-degree polynomial of
the class-conditional probability with unknown coefficients. 

Now, the computational issue is to efficiently form accurate
predictions (as in the binary case~\cite{KS09}, the problem is not
convex).  We now describe a simple strategy for simultaneously
learning the weights as well as the link function, which not only
improves the square loss at every step, but also converges to the
optimal answer quickly.  The strategy maintains two sets of weights,
$\xweights{t} \in \R^{k\times d}$ and $\calweights{t} \in \R^{k\times
  |\linkbasis|}$ and maintains our current predictions $\yiter{t}_i
\in \R^k$ for each data point $i = 1,2,\ldots, n$. After initializing
all the predictions and weights to zero, the updates shown in
Algorithm~\ref{alg:glm-multi-calib} involve two alternating least
squares steps. The first step fits the residual error to $x$ using the
weights $W_t$. This The second step then fits $y$ to the functions of
$\yiter{t}$'s, i.e. to $G(\yiter{t})$. Finally, we project onto the
unit simplex in order to obtain the new predictions, which can only
decrease the squared error and can be done in $O(k)$
time~\cite{DuchiSSSiCh08}.

In the context of the examples of $G_i$ mentioned above, the algorithm
boils down to predicting the conditional probability of $Y= i$ given
$x$, based not only on $x$, but also on our current predictions for
all the classes (and higher degree polynomials in these
predictions)\footnote{The alternating least-squares update in this
  context are also quite reminiscent of CCA.}. 

For the analysis of Algorithm~\ref{alg:glm-multi-calib},
we focus on the noiseless case to understand the optimization issues.
Analyzing the statistical issues, where there is noise, can be
handled using ideas in \cite{KS09,KakadeKKS11}. 

\begin{theorem}
  Suppose that $y_i=g(\Wopt x_i)$ and that the link function $g=\nabla
  \linkloss$ satisfies the Lipschitz and strong monotonicity
  conditions~\eqref{eqn:lipschitz-multi}
  and~\eqref{eqn:multi-link-strong} with constants $L$ and $\mu$
  respectively. Suppose also that $\nabla \linkloss(0) =
  \ones/k$. Using the (calibrated) Least Squares updates
  (Algorithm~\ref{alg:glm-multi-calib}) with $W_0= 0$, for all $t
  = 1,2,\ldots$ we have the bound
\[  
\frac{1}{n} \sum_{i=1}^n \norm{\yiter{t}_i - y_i}_2^2 \leq
  \frac{22\cond_{\linkloss}^2}{t}
\]
where $\cond_{\linkloss} = L/\mu$.
\label{thm:multi-calib}
\end{theorem}

We again emphasize the fact that the
updates~\eqref{eqn:alg-calib-multi1} and~\eqref{eqn:alg-calib-multi2}
only require the solution of least-squares problems in a similar
spirit as Algorithm~\ref{alg:glm-multi-full}. Finally, we note that
the rules to compute predictions in our
updates(~\eqref{eqn:alg-calib-multi1} and
~\eqref{eqn:alg-calib-multi2}) require previous predictions (i.e. the
learned model is not \emph{proper} in that it does not actually
estimate $g$, yet it is still guaranteed to make accurate
predictions).

\subsection{Scalable Variants}

\def\GEN{\text{\textsc{GEN}}}
\begin{algorithm}[t]
\begin{algorithmic}
  \REQUIRE data $\{(x_i, y_i)\}$, batch generator $\GEN$, batch size $p$, iterations $T$
\vspace{0.1in}
    \STATE Initialize predictions 
    \[
    \yiter{1}_i = 0
    \]
    \FOR{$t=1,\ldots,T$}
  \STATE Generate $p$ features from the original ones
  \[
      \{\tilde{x}_i\} = \GEN(\{x_i\},p)
  \]
  \STATE Let $W_t$ be the output of Algorithm~\ref{alg:glm-multi-full} or \ref{alg:glm-multi-calib} 
  on the dataset $\{(\tilde{x}_i, y_i-\yiter{t}_i)\}$
   \STATE Update predictions
   \[
       \yiter{t+1}_i  = \yiter{t}_i + W_t\tilde{x}_i
   \]
 \ENDFOR
\end{algorithmic}
\caption{Stagewise Regression
\label{alg:glm-multi-stagewise}}
\end{algorithm}

When the number of features is large, any optimization algorithm that
scales superlinearly with the dimensionality faces serious
computational issues. In such cases we can adopt a block coordinate
descent approach. To keep the presentation fairly general, we assume
that we have an algorithm $\GEN$ that returns a small set of $m$
features, where $m$ is small enough so that fitting models with $m$
features is efficient (e.g. we typically use $m\approx 1000$). The
$\GEN$ procedure can be as simple as sampling $m$ of the original
features (with or without replacement) or more complex schemes such as
random Fourier features~\cite{rahimi2007random}. We call $\GEN$ and
fit a model on the $m$ features using either
Algorithm~\ref{alg:glm-multi-full} or
Algorithm~\ref{alg:glm-multi-calib}.  We then compute residuals and
repeat the process on a fresh batch of $m$ features returned by
$\GEN$. In Algorithm~\ref{alg:glm-multi-stagewise} we provide
pseudocode for this stagewise regression procedure. We stress that
this algorithm is purely a computational convenience. It can be
thought as the algorithm that would result by a block-diagonal
approximation of the second moment matrix $\Sigma$ (not just across
classes, but also groups of features).
Algorithm~\ref{alg:glm-multi-stagewise} bears some resemblance to
boosting and related coordinate descent methods, with the crucial
difference that $\GEN$ is not restricted to searching for the best set
of features. Indeed, in our experiments $\GEN$ is either sampling from
the features without replacement or randomly projecting the data in
$m$ dimensions and transforming each of the $m$ dimension by a simple
non-linearity. Despite its simplicity, more work needs to be done to
theoretically understand the properties of this variant as clearly as
those of Algorithm~\ref{alg:glm-multi-full} or
Algorithm~\ref{alg:glm-multi-calib}. Practically, stagewise regression
can have useful regularization properties but these can be subtle and
greatly depend on the $\GEN$ procedure. In text classification, for
example, fitting the most frequent words first leads to better models
than fitting the least frequent words first.

\section{Experiments}

We consider four datasets MNIST, CIFAR-10, 20 Newsgroups, and RCV1
that capture many of the challenges
encountered in real-world learning tasks.  We believe that the lessons
gleaned from our analysis and comparisons of performance on these
datasets apply more broadly.

For MNIST, we compare our algorithms
with a variety of standard algorithms.  Both in terms of
classification accuracy and optimization speed, we achieve close to
state of the art performance among permutation-invariant methods 
($1.1\%$ accuracy, improving upon methods such as the ``dropout''
neural net).  For CIFAR-10, we also obtain nearly state of the art 
accuracy ($>85\%$) using standard features.  Here, we emphasize that 
it is the computational efficiency
of our algorithms which enables us to achieve higher accuracy without
novel feature-generation.

The story is rather different for the two text datasets, where the
performance of these stagewise methods is less competitive with
online approaches, though we do demonstrate substantial reduction 
in error rate in one of the problems. As we discuss below, the statistical
properties of these text datasets (which differ significantly from
those of the image datasets) strongly favor online approaches.

\begin{figure*}[!t]
  \centering
  \setlength{\tabcolsep}{2pt}
  \begin{tabular}{cc}
  \includegraphics[width=0.35\linewidth]{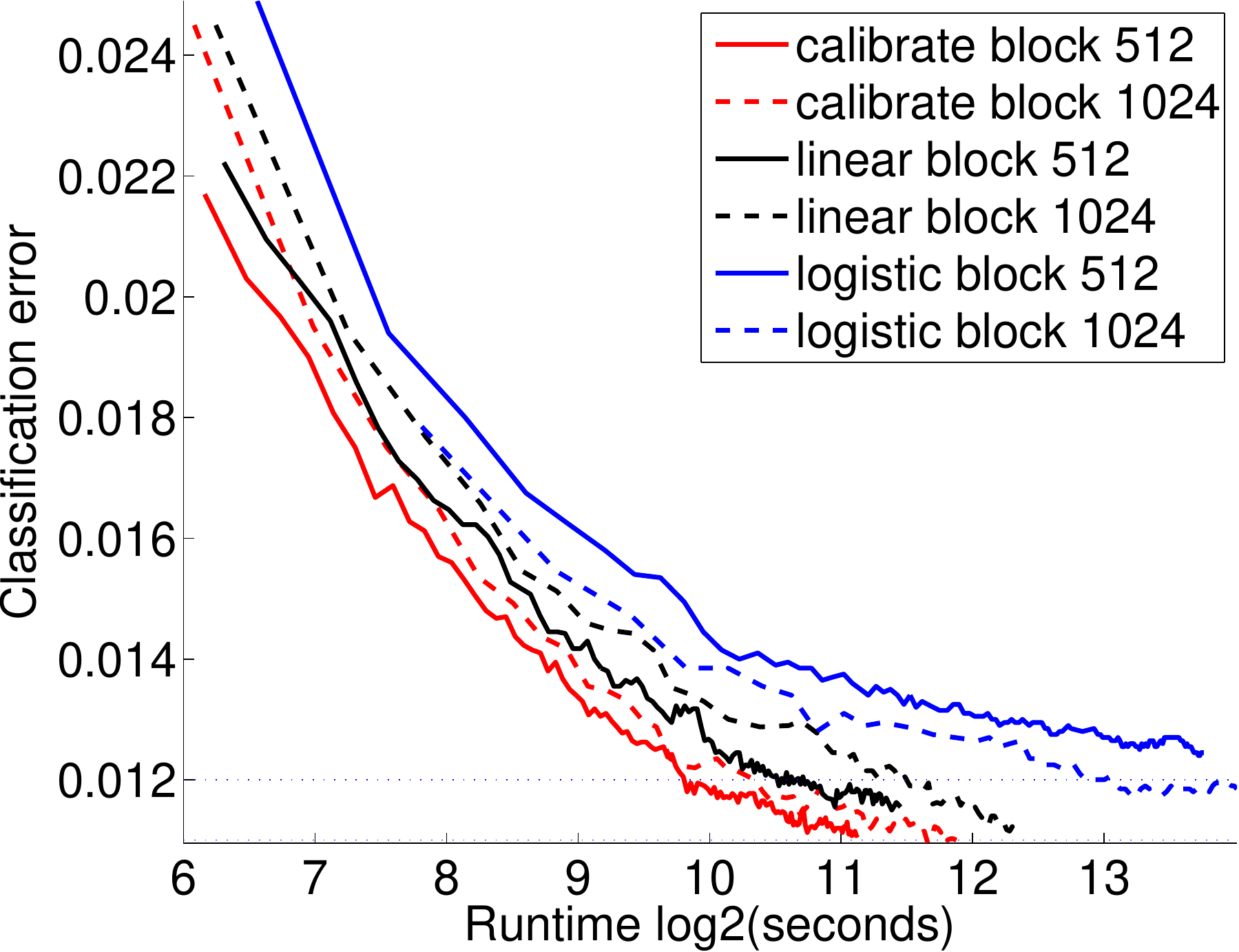}
  &
  \includegraphics[width=0.4\linewidth]{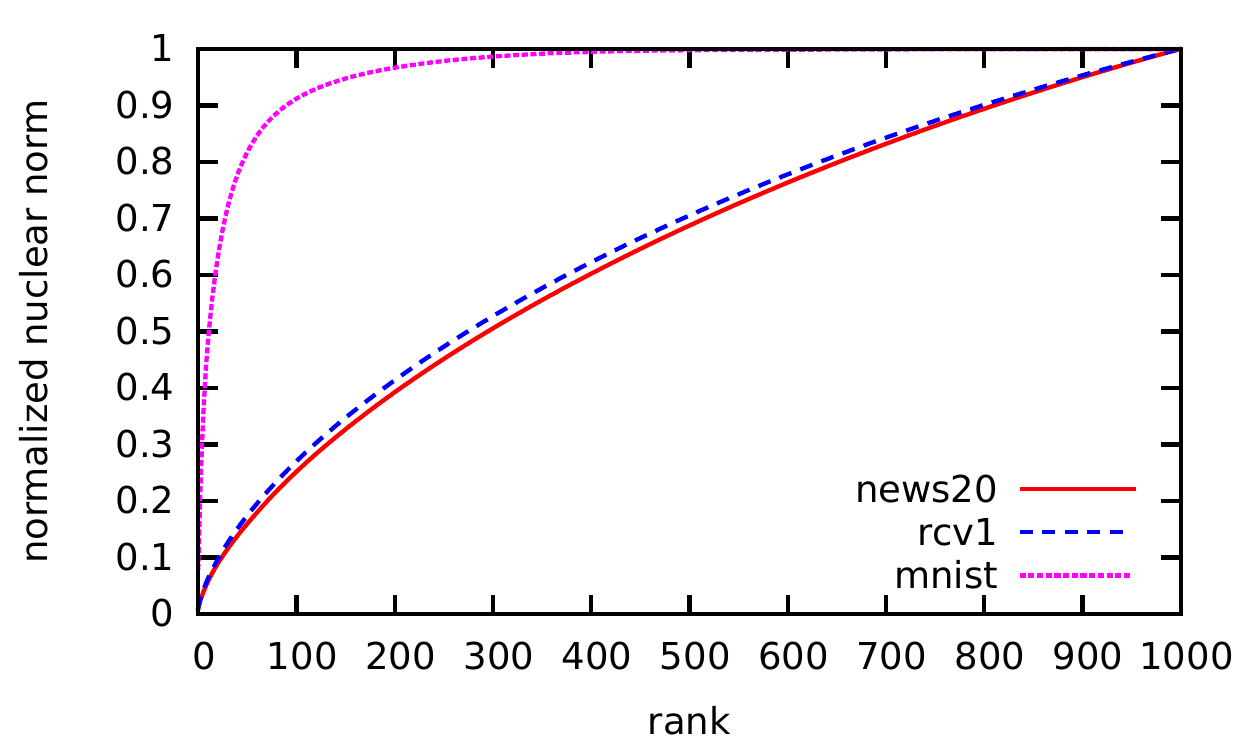}\\	
  (a) Time vs Error & (b) Text vs Vision Data Spectrum  
  \end{tabular}
  \caption{ {\footnotesize (a) Runtime versus error for different
            variants of our algorithms; (b) The fraction of the sum of the top 1000 singular values 
that is captured by the top $x$ singular values.\/} }
  \label{figure:mnist_spec}
  \vspace{-4mm}
\end{figure*}

\subsection{MNIST}
\label{sec:mnist}

Nonlinear classifiers are needed to achieve state-of-the-art
performance in MNIST dataset.  Although MNIST dataset only contains
60K data points (small by modern standards), the requirement
for nonlinear features make this dataset computationally
challenging. For instance, a nonlinear support vector machine with a
Gaussian RBF kernel needs to manipulate a 60K$\times$60K kernel
matrix. This will require hundreds of Gigabytes of memory, not available
on most modern desktop machines. Hence we use an explicit feature
representation and train our classifiers in the primal
space. Specifically we construct random fourier features which are
known to approximate the Gaussian kernel $k(x,x')=\exp(-\|x - x'\|^2 /
s)$~\cite{rahimi2007random}, though as discussed in the appendix
various other methods to construct random low degree polynomials are
also effective. 

We start by comparing linear and logistic regression using
Algorithm~\ref{alg:glm-multi-full}, as well as the calibration variant
of Algorithm~\ref{alg:glm-multi-calib}. For the calibration variant,
we use a basis $G(y)$ consisting of $y$, $y^2$ and $y^3$ (applied
elementwise to vector $y$). We compare these algorithms on raw pixel
features, as well as small number of random Fourier features described
above. As seen in Table~\ref{tb:linlog}, the performance of logistic
and calibrated variants seem similar and consistently superior to
plain linear regression.


Next, we move to improving accuracy by using the stagewise approach of
Algorithm~\ref{alg:glm-multi-stagewise}, which allows us to scale up
to larger number of random Fourier features. Concretely, we fit blocks
of features (either 512 and 1024) with
Algorithm~\ref{alg:glm-multi-stagewise} with three alternative update
rules on each stage: linear regression, calibrated linear regression,
and logistic regression (with 50 inner loops for the logistic
computations). Here, our calibrated linear regression is the simplest
one: we only use the previous predictions as features in our new batch
of features.

Our next experiment demonstrates that all three (extremely simple and
parameter free) algorithms quickly achieve state of the art
performance. Figure~\ref{figure:mnist_spec}(a) shows the relation
between feature block size, classification test error, and runtime for
these algorithm variants. Importantly, while the linear (and linear
calibration) algorithms do not achieve as low an error for a fixed
feature size, they are faster to optimize and are more effective
overall.

\begin{table}[b] 
\vspace{-3mm} 
\caption{ {\footnotesize
    Linear Regression vs. logistic regression vs. (polynomial) calibration.
    For the polynomial calibration, we refit our
    predictions with $\hat y$, $\hat y^2$ and $\hat y^3$. \/}}
\label{tb:linlog}
\centering
\begin{tabular}{|l|c|c|c|c|c|}
\hline
Algorithm & Linear & Logistic & (poly.) Calibration \\
\hline
Raw pixels & 14.1\% & 7.8\% & 8.1\% \\
4000 dims &  1.83 \% & 1.48 \% & 1.54 \% \\
8000 dims & 1.48\% & 1.33\%& 1.36\%\\
\hline
\end{tabular}
\end{table}

Notably, we find that (i) linear regression achieves better runtime
and error trade-off, even though for a fixed number of features linear
regressions are not as effective as logistic regression (as we see in
Figure~\ref{figure:mnist_mle}(a)). (ii) relatively small size feature blocks
provide better runtime and error trade-off (blocks of size 300 provide
further improvements).
(iii) the linearly calibrated regression works better than the vanilla
linear regression. 

\begin{figure*}[!t]
  \centering
  \setlength{\tabcolsep}{2pt}
    \begin{tabular}{cc}
      \includegraphics[width=0.48\linewidth]{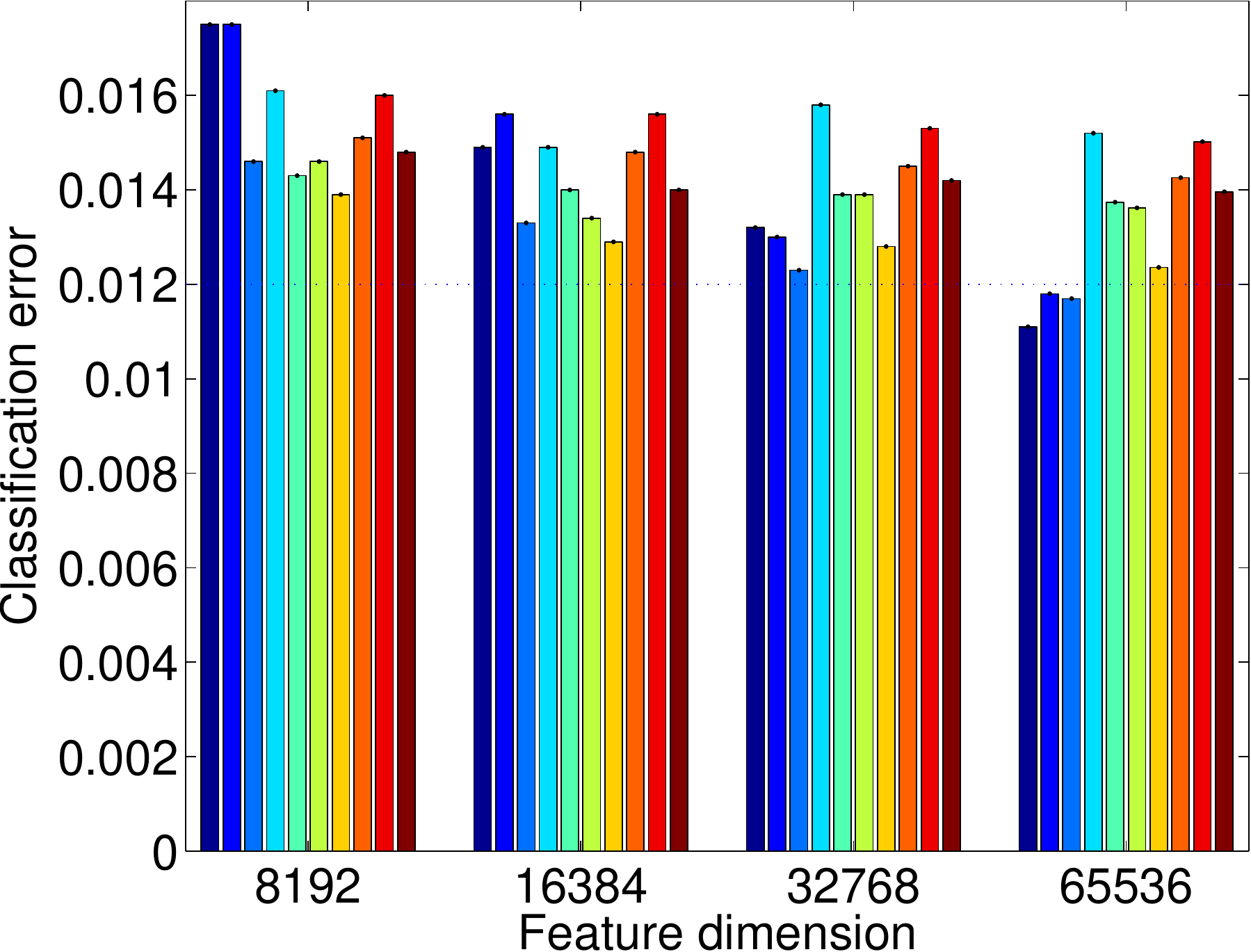}
        &
      \includegraphics[width=0.46\linewidth]{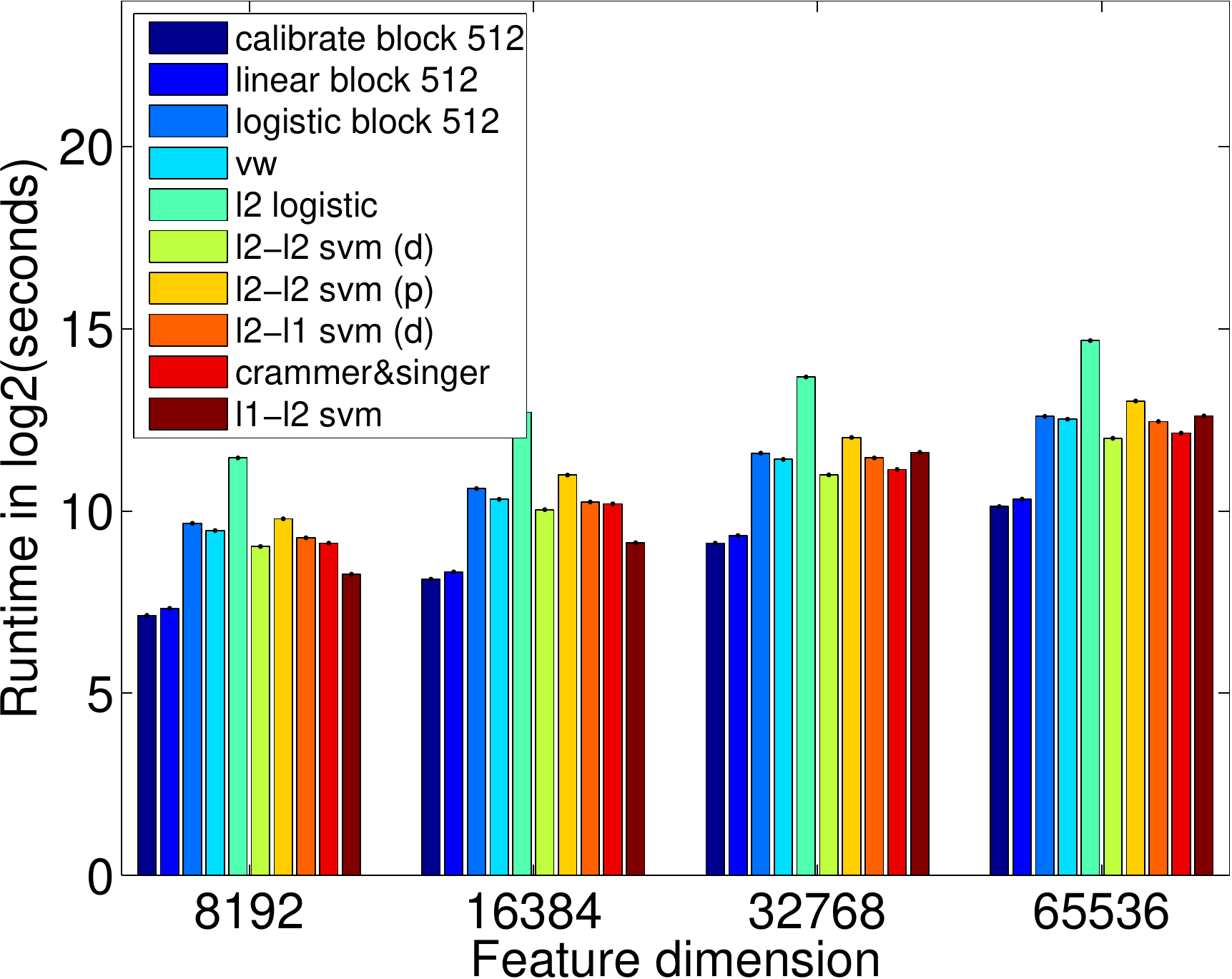}\\
       (a) Comparable Errors & (b) Time Comparisons
    \end{tabular}
    \vspace{-4mm}
    \caption{ {\footnotesize (a) Error comparison between our
            variants, VW and six variants of Liblinear: L2-regularized
            logistic regression, L2-regularized L2-loss support vector
            classification (dual), L2-regularized L2-loss support vector
            classification (primal), L2-regularized L1-loss support vector
            classification (dual), multi-class support vector classification
            by Crammer and Singer, L1-regularized L2-loss support vector
            classification.
            (b) Runtime comparison in {\bf log\/} time. }}
    \label{figure:mnist_mle}
    \vspace{-4mm}
\end{figure*}

We also compare these three variants of our approach to other
state-of-the-art algorithms in terms of classification test error and
runtime (Figure~\ref{figure:mnist_mle}(a) and (b)). Note the
\emph{logarithmic} scaling of the runtime axes.  The comparison
includes VW~\footnote{\url{http://hunch.net/~vw/}}, and six algorithms
implemented in Liblinear~\cite{fan2008Liblinear} (see figure
caption). We took care in our attempts to time these algorithms to
reflect their actual computation time, rather than their loading of
the features (which can be rather large, making it extremely time
consuming to run these experiments); our stagewise algorithms generate
new features on the fly so this is not an issue. See the appendix for
further discussion.

From Figure~\ref{figure:mnist_mle} (a) and (b), our logistic algorithm
is competitive with all the other algorithms, in terms of it's
accuracy (while for lower dimensions the naive linear methods fared a
little worse). All of our algorithms were substantially faster.

Finally, the models produced by our methods drive the classification
test error down to 1.1\% while none of the competitors achieve this
test error. Runtime wise, our method is extremely fast for the linear
regression and calibrated linear regression variants, which are
consistently at least 10 times faster than the other highly optimized
algorithms. This is particularly notable given the simplicity of this
approach.

%

%

\subsection{CIFAR-10}

The CIFAR-10 dataset is a more challenging dataset, where many image
recognition algorithms have been tested (primarily illustrating
different methods of feature generation; our work instead focusses on
the optimization component, given a choice of features).  The neural
net approaches of ``dropout'' and ``maxout'' algorithms of
\cite{hinton2012improving,goodfellow:maxout} provide the best reported
performance of 84\% and 87\%, without increasing the size of the
dataset (through jitter or other transformations).  We are able to
robustly achieve over 85\% accuracy with linear regression on standard
convolution features (without increasing the size of the dataset
through jitter, etc.), illustrating the advantage that improved
optimization provides.

Figure~\ref{figure:cifar} illustrates the performance when we use two
types of convolutional features: features generated by convolving the
images by random masks, and features generated by convolving with
K-means masks (as in \cite{DBLP:journals/jmlr/CoatesNL11}, though we
do \emph{not} use contrast normalization). 

We find that using only relatively few filters (say about 400), along
with polynomial features, are sufficient to obtain over $80\%$
accuracy extremely quickly.  Hence, using the thousands of generated
features, it is rather fast to build multiple models with disjoint
features and model average them, obtaining extremely good performance.

\begin{figure}[!t]
\hspace{-5mm}
\includegraphics[ width=3.6in]{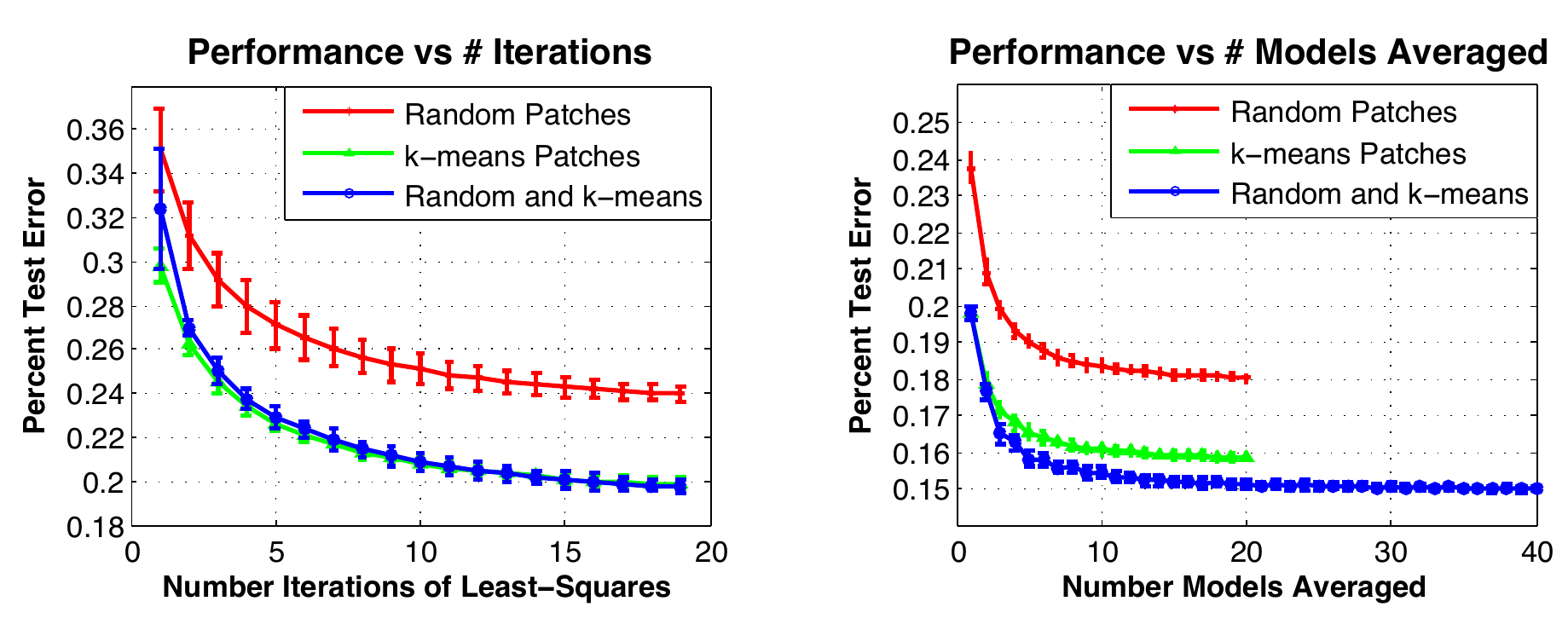}\label{fig:cifar} \\
\vspace{-7mm}
\begin{center}(a)~~~~~~~~~~~~~~~~~~~~~~~~~~~~~~~~~~~~~~~~~~~~~~~~~~~~(b)\end{center}
\vspace{-3mm}
\caption{CIFAR-10 results using two
types of convolutional features: (a) features generated by convolving
the images by random masks, and (b) features generated
by convolving with K-means masks.}
\label{figure:cifar}
\vspace{-3mm}
\end{figure}

\subsection{Well-Conditioned Problems}

We now examine two popular multiclass text datasets:
20~newsgroups\footnote{http://archive.ics.uci.edu/ml/datasets/Twenty+Newsgroups}
(henceforth NEWS20), which is a 20 class dataset and a four class version 
of Reuters Corpus Volume~1~\cite{lewis2004rcv1} (henceforth RCV1). We use a standard
(log) term frequency representation of the data, discussed in the
appendix.  These data pose rather different challenges than our vision
datasets; in addition to being sparse datasets, they are extremely
well conditioned.

The ratio of the 2nd singular value to the 1000th one (as a proxy
for the condition number) and is 19.8 for NEWS20 and 14 for
RCV1. In contrast, for MNIST, this condition number is about 
72000 (when computed with 3000 random Fourier features). 
Figure~\ref{figure:mnist_spec}(b) shows the normalized spectrum for
the three data matrices.

As expected, online first-order methods (in particular VW) fare far
more favorably in this setting, as seen in Table~\ref{fig:text}. We
use a particular greedy procedure for a our stagewise ordering (as
discussed in the appendix, though random also works well). Note that
this data is well suited for online methods: it is sparse (making the
online updates cheap) and well conditioned (making the gap in
convergence between online methods and second order methods small).

Ultimately, an interesting direction is developing hybrid approaches
applicable to both cases.

\begin{table}
\centering 
\caption{{\footnotesize Running times and test errors
in text datasets for VW, Liblinear, and Stagewise regression}}
\label{fig:text}
    \begin{tabular}{|c|c|c|c|c|}
        \hline
               & \multicolumn{2}{|c|}{NEWS20} & \multicolumn{2}{|c|}{RCV1}\\
        \hline
        Method & Time   & \%Error & Time & \%Error\\
        \hline
        VW        &  2.5 & 12.4 & 2.5 & 2.75 \\
        Liblinear &  27  & 13.8 & 120 & 2.73 \\
        Stagewise &  40  & 11.7 & 240 & 2.77 \\
        \hline
    \end{tabular}
\vspace{-3mm}    
\end{table}

\section{Discussion} 

In this paper, we present a suite of fast and simple algorithms for
tackling large-scale multiclass prediction problems. We stress that
the key upshot of the methods developed in this work is their
conceptual simplicity and ease of implementation. Indeed these
properties make the methods quite versatile and easy to extend in
various ways. We showed an instance of this in
Algorithm~\ref{alg:glm-multi-calib}. Similarly, it is straightforward
to develop accelerated variants~\cite{Nesterov09}, by using the
distances defined by the matrix $\covhat$ as the prox-function in
Nesterov's work. These variants enjoy the usual improvements of
$\order(1/t^2)$ iteration complexity in the smooth and
$\sqrt{\cond_{\linkloss}}$ dependence in the strongly convex setting,
while retaining the metric-free nature of
Algorithm~\ref{alg:glm-multi-full}. 

It is also quite easy to extend the algorithm to multi-label settings,
with the only difference being that the vector $y$ of labels now lives
on the hypercube instead of the simplex. This only amounts to a minor
modification of the projection step in
Algorithm~\ref{alg:glm-multi-calib}. 

Overall, we believe that our approach revisits many old and deep ideas
to develop algorithms that are practically very effective. We believe
that it will be quite fruitful to understand these methods better both
theoretically and empirically in further research.

\begin{footnotesize}
\bibliographystyle{icml2014}
\bibliography{bib}
\end{footnotesize}

\clearpage
\newpage
\onecolumn

\appendix
\section{Appendix}

\subsection{Proofs}

\begin{proof-of-theorem}[\ref{thm:glm-multi-newton}]
  
  We start by noting that the Lipschitz and strong monotonicity
  conditions on $\nabla \linkloss$ imply the smoothness and strong
  convexity of the function $\linkloss(u)$ as a function $u \in
  \R^k$. In particular, given any two matrices $W_1, W_2 \in
  \R^{k\times d}$, we have as the following quadratic upper bound as a
  consequence of the Lipschitz condition~\eqref{eqn:lipschitz-multi} 

  \begin{equation*}
  \linkloss(W_1x) \leq \linkloss(W_2x) + \ip{\nabla \linkloss(W_2 x^T )
    x}{W_1 - W_2} + \frac{L}{2} \norm{W_1x - W_2 x}_2^2.
  \end{equation*}
  The strong monotonicity condition~\eqref{eqn:multi-link-strong}
  yields an analogous lower bound 

  \begin{equation*}
  \linkloss(W_1x) \geq \linkloss(W_2x) + \ip{\nabla \linkloss(W_2 x^T )
    x}{W_1 - W_2} + \frac{\mu}{2} \norm{W_1x - W_2 x}_2^2.
  \end{equation*}

  In order to proceed further, we need one additional piece of
  notation. Given a positive semi-definite matrix $M \in \R^{d \times
    d}$, let us define 
    \[
        \norm{W}_M = \sum_{i=1}^k {W^{(i)}}^T M W^{(i)}, 
    \]
    where $W^{(i)}$ is the $i$-th column of $W$, to be a Mahalanobis norm on
  matrices. Let us also recall the definition of the
  matrix $\covhat$ from Algorithm~\ref{alg:glm-multi-full}. Then
  adding the smoothness condition over the examples $i=1,2,\ldots,n$
  yields the following conditions on the sample average loss under
  condition~\eqref{eqn:lipschitz-multi}:

  \begin{equation*}
    \lossn(W_1) \leq \lossn(W_2) + \ip{\nabla \lossn(W_2)}{W_1 - W_2}
    + \frac{L}{2} \norm{W_1 - W_2}_{\covhat}^2. 
  \end{equation*}
  This implies that our objective function $\lossn$ is $L$-smooth in
  the metric induced by $\covhat$, and the update
  rule~\eqref{eqn:glm-multi-newton} corresponds to gradient descent on
  $\lossn$ under this metric with a step-size of $1/L$. The first part
  of the theorem now follows from Corollary 2.1.2 of
  Nesterov~\cite{Nesterov04}.

  As for the second part, we note that under the strong monotonicity
  condition, we have the lower bound 

  \begin{equation*}
    \lossn(W_1) \geq \lossn(W_2) + \ip{\nabla \lossn(W_2)}{W_1 - W_2}
    + \frac{\mu}{2} \norm{W_1 - W_2}_{\covhat}^2. 
  \end{equation*}
  Hence the objective $\lossn$ is $\mu$-strongly convex and $L$-smooth
  in the metric induced by $\covhat$. The result is now a consequence
  of Theorem 2.1.15 of Nesterov~\cite{Nesterov04}. 
\end{proof-of-theorem}

We now provide the proof of Theorem~\ref{thm:multi-calib}. First, a
little more on our assumption on $g^{-1}$. By convex duality, this
inverse exists and if $\linkloss$ is a closed, convex function then
$(\nabla \linkloss)^{-1} = \nabla \linkloss^{*}$, where
$\linkloss^{*}$ is the Fenchel-Legendre conjugate of
$\linkloss$. Throughout this section, assume that $\nabla \linkloss$
is $L$-Lipschitz continuous. By standard duality results regarding
strong-convexity and smoothness, this implies that the conjugate
$\linkloss$ is $1/L$-strongly convex. Specifically, we have the useful
inequality

\begin{equation}
  \ip{\nabla \linkloss^{*}(u) - \nabla \linkloss^{*}(v)}{u - v} \geq
  \frac{1}{L} \norm{u - v}_2^2, \quad \mbox{for all}~~u,v \in \R^k.
\label{eqn:multi-strong-dual}
\end{equation}

Similarly, due to our assumption about the strong monotonicity of
$\nabla \linkloss$, it is the case that $\nabla \linkloss^{*}$ is
Lipschitz continuous and satisfies

\begin{equation}
  \ip{\nabla \linkloss^{*}(u) - \nabla \linkloss^{*}(v)}{u - v} \leq
  \frac{1}{\mu} \norm{u - v}_2^2, \quad \mbox{for all}~~u,v \in \R^k.
\label{eqn:multi-smooth-dual}
\end{equation}
As a specific consequence, note that it is natural to assume that
$\nabla \linkloss(0) = \ones/k$, where $\ones$ is the all ones
vector. That is the expectation is uniform over all the labels when
the weights are zero. Under this condition, it is easy to obtain as a
consequence of Equation~\ref{eqn:multi-smooth-dual} that

\begin{align}
  \nonumber \norm{\nabla \linkloss^{*}(u)}_2 &= \norm{\nabla
    \linkloss^{*}(u) - \nabla \linkloss^{*}(\ones/k)}_2\\
  &\leq \frac{1}{\mu} \norm{u - \ones/k}_2 \leq \frac{1}{\mu}
  \left(\norm{u}_2 + \frac{1}{\sqrt{k}}\right).
  \label{eqn:multi-lip-dual}
\end{align}

Together with these facts, we now proceed to establish
Theorem~\ref{thm:multi-calib}. 

\begin{proof-of-theorem}[\ref{thm:multi-calib}]

We will use $\ynoclip{t}_i = \calweights{t} G(\ziter{t}_i)$ to denote
the predictions at each iteration before the clipping operation. For
brevity, we use $\yitermat{t} \in \R^{n\times k}$ to denote the matrix
of all the predictions at iteration $t$, with a similar version
$\zitermat{t}$ for $\ziter{t}$

The following basic properties of linear regression that are
helpful. By the optimality conditions for $\xweights{i}$ and
$\calweights{i}$, 

\begin{align}
  \nonumber \sum_{i=1}^n (\ziter{t}_i - y_i) x_i^T  &=
  \mathbf{0}_{k\times d}, \quad \mbox{and}\\
  \sum_{i=1}^n (\ynoclip{t}_i - y_i) G(\ziter{t-1}_i)^T  &=
  \mathbf{0}_{|\linkbasis|\times d}. 
  \label{eqn:normal}
\end{align}

In particular, multiplying the first equality with the optimal weight
matrix $\Wopt$ yields 

\begin{equation*}
  \ip{\sum_{i=1}^n (\ziter{t}_i - y_i) x_i^T }{\Wopt} = 0.
\end{equation*}

Rearranging terms and recalling that $\nabla \linkloss(\Wopt x_i) =
y_i$ due to the generative model~\eqref{eqn:glm-multi} further allows
us to rewrite

\begin{equation}
  \ip{\sum_{i=1}^n (\ziter{t}_i - y_i)}{\nabla \linkloss^*(y_i)} = 0.
  \label{eqn:normal-to-optinv}
\end{equation}

Combining this with our earlier
inequality~\eqref{eqn:multi-strong-dual} further yields 

\begin{align}
  \sum_{i=1}^n \ip{\ziter{t}_i - y_i}{\nabla \linkloss^*(\ziter{t}_i)}
  &= \sum_{i=1}^n \ip{\ziter{t}_i - y_i}{\nabla
    \linkloss^*(\ziter{t}_i) - \nabla \linkloss^*(y_i)} \geq
  \frac{1}{L} \sum_{i=1}^n \norm{\ziter{t}_i - y_i}_2^2. 
  \label{eqn:residual-lb}
\end{align}

Having lower bounded this inner product term, we obtain an upper bound
on it which will complete the proof for convergence of the algorithm. Note
that $\calweights{t}$ minimizes the
objective~\eqref{eqn:alg-calib-multi2}. Since $\nabla \linkloss^* \in
\mbox{lin}(\linkbasis)$, we have for any constant $\beta \in \R$

\begin{equation*}
  \sum_{i=1}^n \norm{\ynoclip{t}_i - y_i}_2^2 \leq \sum_{i=1}^n
  \norm{\ziter{t}_i - y_i - \beta \nabla
    \linkloss^{*}(\ziter{t})}_2^2. 
\end{equation*}
We optimize over the choices of $\beta$ to obtain the best
inequality above, which yields the error reduction as 

\begin{align}
  \sum_{i=1}^n \norm{\ynoclip{t}_i - y_i}_2^2 &\leq \sum_{i=1}^n
  \norm{\ziter{t}_i - y_i}_2^2 - \frac{\ip{\zitermat{t} - Y}{\nabla
      \linkloss^*(\zitermat{t})}^2}{\norm{\nabla
      \linkloss^*(\zitermat{t})}_F^2}.
  \label{eqn:residual-ub}
\end{align}
We now proceed to upper bound the denominator in the second term in
the right hand side of the above bound. Note that from
Equation~\ref{eqn:multi-lip-dual}, we have the upper bound 

\begin{align*}
  \norm{\nabla \linkloss^{*}(\ziter{t}_i)}_2^2 &\leq \frac{2}{\mu^2}
  \left( \norm{\yiter{t}_i}_2^2 + \frac{1}{k}\right)\\
  &\leq \frac{2}{\mu^2 k} + \frac{4}{\mu^2} \left( \norm{\ziter{t}_i -
    y_i}_2^2 + \norm{y_i}_2^2\right)\\
  &\leq \frac{2}{\mu^2k} + \frac{4}{\mu^2} \left( \norm{\ziter{t}_i -
    y_i}_2^2 + 1\right),
\end{align*}
where the final inequality follows since $y_i$ has a one in precisely
one place and zeros elsewhere. Adding these inequalities over the
examples, we further obtain 

\begin{align*}
  \sum_{i=1}^n \norm{\nabla \linkloss^{*}(\ziter{t}_i)}_2^2 &\leq
  \frac{2n}{\mu^2k} + \frac{4}{\mu^2} \left(
  \sum_{i=1}^n\norm{\ziter{t}_i - y_i}_2^2 + n\right)\\ 
  &\leq \frac{6n}{\mu^2} + \frac{4}{\mu^2} \sum_{i=1}^n
  \norm{\yiter{t}_i - y_i}_2^2,
\end{align*}
where the last step is an outcome of solving the regression problem in
the step~\eqref{eqn:alg-calib-multi1}. Finally, observe that
$\yiter{t}$ is a probability vector in $\R^k$ as a result of the
clipping operation, while $y_i$ is a basis vector as before. Taking
these into account, we obtain the upper bound

\begin{align}
  \sum_{i=1}^n \norm{\nabla \linkloss^{*}(\ziter{t}_i)}_2^2 \leq
  \frac{22nk}{\mu^2}.
  \label{eqn:upper-norm-ziter}
\end{align}

We are almost there now. Observe that we can substitute this upper
bound into our earlier inequality~\eqref{eqn:residual-ub} and obtain 

\begin{align}
  \sum_{i=1}^n \norm{\ynoclip{t}_i - y_i}_2^2 &\leq \sum_{i=1}^n
  \norm{\ziter{t}_i - y_i}_2^2 - \frac{\mu^2}{22nk}\ip{\zitermat{t} -
    Y}{\nabla \linkloss^*(\zitermat{t})}^2.
  \label{eqn:residual-ub-better}
\end{align}

We can further combine this inequality with the lower
bound~\eqref{eqn:residual-lb} and obtain 

\begin{align}
  \sum_{i=1}^n \norm{\ynoclip{t}_i - y_i}_2^2 &\leq \sum_{i=1}^n
  \norm{\ziter{t}_i - y_i}_2^2 - \frac{\mu^2}{22nkL^2}
  \left(\sum_{i=1}^n \norm{\ziter{t}_i - y_i}_2^2 \right)^2.
  \label{eqn:ziter-to-noclip}
\end{align}

This would yield a recursion if we could replace the term
$\sum_{i=1}^n \norm{\ynoclip{t}_i - y_i}_2^2$ with $\sum_{i=1}^n
\norm{\ziter{t+1}_i - y_i}_2^2$. This requires the use of two critical
facts. Note that $\yiter{t}_i$ is a Euclidean projection of
$\ynoclip{t}_i$ onto the probability simplex and $y_i$ is an element
of the simplex. Consequently, by Pythagoras theorem, it is easy to
conclude that 

\begin{equation*}
  \norm{\yiter{t}_i - y_i}^2 \leq \norm{\ynoclip{t}_i - y_i}_2^2. 
\end{equation*}
Furthermore, the regression update~\eqref{eqn:alg-calib-multi1}
guarantees that we have 

\begin{align*}
  \sum_{i=1}^n \norm{\ziter{t+1} - y_i}_2^2 \leq \sum_{i=1}^n
  \norm{\yiter{t}_i - y_i}_2^2 \leq  \sum_{i=1}^n \norm{\ynoclip{t}_i
    - y_i}_2^2. 
\end{align*}

Combining the update with earlier bound~\eqref{eqn:ziter-to-noclip},
we finally obtain the recursion we were after:

\begin{align}
  \sum_{i=1}^n \norm{\ziter{t+1}_i - y_i}_2^2 &\leq \sum_{i=1}^n
  \norm{\ziter{t}_i - y_i}_2^2 - \frac{\mu^2}{22nkL^2}
  \left(\sum_{i=1}^n \norm{\ziter{t}_i - y_i}_2^2 \right)^2.
  \label{eqn:ziter-recur}
\end{align}

Let us define $\epsilon_t = \frac{1}{n} \sum_{i=1}^n \norm{\ziter{t}_i
  - y_i}_2^2$. Then the above recursion can be simplified as 

\begin{align*}
  \epsilon_{t+1} &\leq \epsilon_t - \frac{\mu^2}{22n^2kL^2}
  \left(\sum_{i=1}^n \norm{\ziter{t}_i - y_i}_2^2 \right)^2\\
  &= \epsilon_t - \frac{\mu^2}{22kL^2} \epsilon_t^2. 
\end{align*}

It is straightforward to verify that the recursion is satisfied by
setting $\epsilon_t = 22kL^2/(\mu^2t) =
22\cond_{\linkloss}^2/t$. Lastly, observe that as a consequence of the
update~\eqref{eqn:alg-calib-multi2} and the contractivity of the
projection operator, we also have 

\begin{equation*}
  \frac{1}{n} \sum_{i=1}^n \norm{\yiter{t}_i - y_i}_2^2 \leq
  \frac{1}{n} \sum_{i=1}^n \norm{\ziter{t}_i - y_i}_2^2 \leq
  \frac{22\cond_{\linkloss}^2}{t}, 
\end{equation*}
which completes our proof. 

\end{proof-of-theorem}

\subsection{Experimental Details}

\subsubsection{MNIST}

We utilize random Fourier features after PCA-ing the data down to $50$
dimensions. This projection alone is a considerable speed improvement
(for all algorithms) with essentially no loss in accuracy. We then
applied the random features as described in \cite{rahimi2007random}.
We should note that our reliance on these features is not critical; both random
low degree polynomials or random logits (with weights chosen from
Gaussian distribution of appropriate variance, so the features have reasonable sensitivity)
give comparable performance. Notably, the random logits seem to need
substantially fewer features to get to the same accuracy level (maybe by a
factor of 2).

The kernel bandwidth $s$ is set using the folklore ``median trick'', 
i.e.,~$s$ is the median pairwise distance between training points.

For VW, we
first generated its native binary input format. VW uses a
separate thread for loading data, so that the data loading time is
negligible. For Liblinear, we used a modified version which can
directly accept dense input features from MATLAB. As with the rest of 
the experiments, our methods are implemented in MATLAB. 
For all methods, the computation for model fitting
is done in a single 2.4 GHz processor.

\subsubsection{20 Newsgroups and RCV1}

Our theoretical understanding suggests that when the condition number
is small then we expect first
order methods to be highly
effective. Furthermore, these methods enjoy an extra computational
advantage when the data is sparse, as the computation of the gradient
updates are linear time in the sparsity level. Here, as expected,
methods that ignore second order information are faster than
stagewise procedures.

We used unigram features with 
with $\log$-transformed term frequencies as values. We also 
removed all words that appear fewer than 3 times on the training
set. For RCV1 the task was to predict whether a news story 
should be classified as ``corporate'', ``economics'', ``government'', 
or ``markets''. Stories belonging to more than one category were 
removed. For RCV1, we switched the roles of training and test folds with
665 thousand and 20 thousand examples for training and testing respectively.
Tokens in RCV1 had already been stemmed and stopwords
had been removed resulting in about 20000 features.
For NEWS20 we did not perform such 
preprocessing leading to about 44000 features. We split the data into 
15000 for training and the rest for testing.

%

For both datasets the bag of words representation is too ``verbose'' to 
handle efficiently with commodity hardware. Therefore, we exclusively 
used the stagewise approach here. In each stage, we picked a subset 
of the original features (i.e., no random projections). To speed up the 
algorithm, we ordered the features by the magnitude of the gradient.

For NEWS20 we used a batch size of 500, 2 passes over the
data and regularization $\lambda=30$. We recomputed the ordering 
between the first and second pass. For RCV1 we used 2
batches of size 2000 and computed the ordering for each batch.
The results are shown in Table~\ref{fig:text}(b). Even though
the stagewise procedure can produce models with the same or better
generalization error than Liblinear and VW, it is not as efficient. 
This is again no surprise since when the condition number is small, 
methods based on stochastic gradient, such as VW, are optimal~\cite{bottou08tradeoff}.

\end{document}